\documentclass[journal]{IEEEtran}

\usepackage[utf8]{inputenc} 
\usepackage[T1]{fontenc}    
\usepackage{hyperref}       
\usepackage{booktabs}       
\usepackage{amsfonts}       
\usepackage{nicefrac}       
\usepackage{microtype}      
\usepackage{amsmath}		
\usepackage{amsthm}			
\usepackage{amssymb}		
\usepackage{multicol}
\usepackage{graphicx}	
\usepackage{subfig}
\usepackage{tikz}
\usepackage{cite}
\usepackage{mdframed}
\usepackage{mathtools}
\usepackage[font=small,labelfont=bf]{caption}
\usepackage{pgfplots}
\usepgfplotslibrary{fillbetween}
\pgfplotsset{width=7cm,compat=newest}
\usetikzlibrary{bayesnet}
\usetikzlibrary{shapes,arrows,positioning,calc}
\usetikzlibrary{datavisualization}
\usetikzlibrary{datavisualization.formats.functions}

\newtheorem{theorem}{Theorem}
\newtheorem{proto-theorem}{Proto-theorem}
\newtheorem{lemma}{Lemma}
\newtheorem{defn}{Definition}

\newtheorem{cor}{Corollary}

\newcommand{\mc}[1]{\mathcal{#1}}

\newmdtheoremenv{const}{Construction}

\begin{document}
\title{Information Losses in Neural Classifiers from Sampling}

\author{
  		Brandon~Foggo,~\IEEEmembership{Student Member,~IEEE,}
        Nanpeng~Yu,~\IEEEmembership{Senior Member,~IEEE,}
        Jie~Shi,~\IEEEmembership{Student Member,~IEEE,}
        and~Yuanqi~Gao,~\IEEEmembership{Student Member,~IEEE,}
}

\tikzset{
		block/.style = {draw, fill=white, rectangle, minimum height=1em, minimum width=3em},
		tmp/.style  = {coordinate}, 
		sum/.style= {draw, fill=white, circle, node distance=1cm},
		input/.style = {coordinate},
		output/.style= {coordinate},
		pinstyle/.style = {pin edge={to-,thin,black}
		}
	}

\maketitle

\begin{abstract}
\footnote{© 20XX IEEE.  Personal use of this material is permitted.  Permission from IEEE must be obtained for all other uses, in any current or future media, including reprinting/republishing this material for advertising or promotional purposes, creating new collective works, for resale or redistribution to servers or lists, or reuse of any copyrighted component of this work in other works.} This paper considers the subject of information losses arising from the finite datasets used in the training of neural classifiers. It proves a relationship between such losses as the product of the expected total variation of the estimated neural model with the information about the feature space contained in the hidden representation of that model. It then bounds this expected total variation as a function of the size of randomly sampled datasets in a fairly general setting, and without bringing in any additional dependence on model complexity. It ultimately obtains bounds on information losses that are less sensitive to input compression and in general much smaller than existing bounds. The paper then uses these bounds to explain some recent experimental findings of information compression in neural networks which cannot be explained by previous work. Finally, the paper shows that not only are these bounds much smaller than existing ones, but that they also correspond well with experiments. 
\end{abstract}

\section{Introduction}
An estimator is limited to the information that it has about the variable it's estimating. But this information is limited to what the estimator has seen from the samples training it. The full information of a random variable cannot be transferred to an estimator by finite samples - some information is lost. This paper analyzes such losses for neural network classifiers. Analyzing these losses can lead to improved architecture designs and training data selection strategies, and provide explanations for empirical results in machine learning theory. 

The study of these loses as a tool for deep learning theory arose from the attempts to understand neural network behavior through the concept of an information bottleneck \cite{tishby2000information, tishby2015deep}. This theory was later investigated both analytically \cite{achille2017emergence} and experimentally \cite{shwartz2017opening, michael2018on}. They are used, primarily, as an explanatory tool which can act as a supplement to classical statistical learning theory (CSLT), which typically fails to explain the success of deep learning models (for example, deep networks tend to perform better when they have \textit{higher} VC dimension, while CSLT would predict the opposite). We will further discuss the utility of these losses in section \ref{sec:background}, and we will denote this newly arising field of deep learning theory as information theoretic deep learning theory (ITDLT). 

But this theory is still somewhat incomplete. The reader will find that reference \cite{michael2018on} above actually contradicts the others - giving experimental evidence \textit{against} some of the claims established in the earlier works. In particular, ITDLT, as it previously stood, would claim that neural networks should always act as a lossy compressor of the input data - a claim which arises from bounds on information losses that are exponential in the information content of the final hidden layer of the network (while still being smaller than CSLT bounds for larger networks). But experiments show that this is only \textit{sometimes} true. While compression does seem to always occur when using saturating activation functions, like sigmoid and tanh, compression in networks using linear and relu activation functions seems to be more nuanced.  

But instead of abandoning ITDLT, we believe that the theory can be improved in such a way that it explains all of these experiments. Since most contrary evidence to the theory can be traced to those exponential bounds, we hypothesize that these bounds, while tighter than those of CSLT, are still not quite tight enough to account for every experiment. In this paper, we aim to derive bounds which are much tighter than the existing ones. This will make up the bulk of this paper, and can be found in section \ref{sec:bounds}.

With these new bounds, we will be able to explain the experimental discrepancy found in the above literature, giving detail into why \textit{some} situations yield neural network compression, even with relu activation functions, and others do not. For example, in the case of low entropy feature spaces, our bounds show that there is simply not enough information to lose such that compression is beneficial. We will illustrate this concept further in section \ref{subsec:illustration}.

This will lead to a better understanding of the information relationships found in neural networks, and to a better understanding of neural networks in general. This better understanding will allow guided development of network architectures and other algorithms which are theoretically sound. 

In one critical step to achieving these bounds (Theorem \ref{thm:estimated_info_bound}), we decompose information losses as a product of a term that mostly depends on network architecture and a term that mostly depends on the training dataset used to train that architecture. This decomposition can thus be applied to network architecture design and training data selection strategies independently. These aspects of applying this theory will be the subject of future work.

Finally, while these new bounds are much tighter than both CSLT bounds and the old ITDLT bounds, and while they are capable of explaining all experiments in literature, we will see experimentally that these bounds are fairly tighter than they needed to be to achieve our goals. This will be shown experimentally in section \ref{sec:experiments}. 

Section \ref{notation} will address some notations and assumptions that we will use throughout the paper. Section \ref{sec:background} will provide more details into the literary background and motivation of this work. We conclude in section \ref{sec:conclusion}.

\section{Notation and Assumptions} \label{notation}
\begin{figure} [t]
	\centering
	\begin{tikzpicture}[auto, node distance=2cm,>=latex']
	\node[block, name=y](y){$p_Y$};
	\node[block, right of=y](XgY){$p_{X|Y}$};
	\node[block, right of=XgY](ZgX){$p_{Z|X}$};
	\node[block, right of=ZgX](pest){Estimator};
	\node[output, right of = pest](est){};
	\draw [->] (y) -- node{$y$} (XgY);
	\draw [->] (XgY) -- node{$x$} (ZgX);
	\draw [->] (ZgX) -- node{$z$} (pest);
	\draw [->] (pest) -- node{$\tilde{y}$} (est);
	\end{tikzpicture}
	\caption{The classification model assumed in this paper.}
	\label{class_model}
\end{figure}
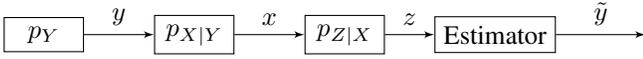
Capital letters denote random variables. Lower case letters describe instances of the corresponding random variable. Figure \ref{class_model} depicts the classification model used in this paper. A class variable $y$ generates a feature vector $x$ according to a fixed (unknown) distribution $\mathbb{P}_{X|Y}$. This feature vector is then fed through a learned distribution $\mathbb{P}_{Z|X}$, which acts as a lossy compressor of $x$. This should be thought of as the hidden layers of a neural network. $z$ is then used to form an estimator of $y$, denoted $\tilde{y}$. We will drop the subscripts on probability distributions when the context is clear.  The calligraphic symbols $\mc{X}$ and $\mc{Y}$ refer to the set of values that $X$ and $Y$ can take on. We assume that $\mc{X}$ is a Polish space such as $\mathbb{R}^d$ and that $\mc{Y}$ is a finite set with the discrete topology.

This model has three variables of interest, $X, Y$ and $Z$ which satisfy the Markov chain $Y-X-Z$. We denote the true model as ${\mathbb{P}_{XYZ} = \mathbb{P}_{X}\mathbb{P}_{Z|X}\mathbb{P}_{Y|X}}$ and consider the case of estimating the conditional probability distribution $\mathbb{P}_{Y|X}$. We denote this estimate as $\hat{\mathbb{P}}_{Y|X}$ and denote the estimated \textit{full} model as ${\hat{\mathbb{P}}_{XYZ} = \mathbb{P}_{X}\mathbb{P}_{Z|X}\hat{\mathbb{P}}_{Y|X}}$. We will use the hat notation for all information theoretic quantities referring to the estimated model. For example:
\begin{equation*}
\hat{I}(X;Y) := \mathbb{E}_{\hat{\mathbb{P}}_{XY}}\left[log ~ \frac{d\hat{\mathbb{P}}_{XY}}{d\left(\mathbb{P}_X \otimes \hat{\mathbb{P}}_{Y}\right)}\right]
\end{equation*}
Finally, we assume that all distributions can be written as density functions such as $p_{XY}(x,y)$. We will occasionally drop the variable-specifying subscript when the context is clear. We will assume that the support of $p(x)$ is all of $\mc{X}$.

\section{Background} \label{sec:background}
\subsection{The Information Bottleneck Principle}
The use of the compressor $p_{Z|X}$ comes from the \textit{Information Bottleneck Problem} \cite{tishby2000information} which attempts to find a variable $Z$ that is \textbf{minimally sufficient} for the input pair of variables $(X,Y)$. The minimal sufficiency of $Z$ refers to the following two properties.  First, $X$ and $Y$ must be conditionally independent given $Z$, or, put in a more enlightening way, ${I(Z;Y)=I(X;Y)}$. And second, for any other sufficient statistic $T$, ${I(X;T) \geq I(X;Z)}$. Intuitively, a minimally sufficient statistic is the most efficient description of $X$ which retains all of the available information about the class variable $Y$. Further reasons that we wish to find a minimally sufficient statistic will become clear in the following sections.

\subsection{Information and Generalization}
We now focus on the reason for caring about the first aspect of finding a minimally sufficient statistic. That is, on finding a variable such that ${I(Z;Y)=I(X;Y)}$, or, in a more relaxed form, at least ensuring finding one such that $I(Z;Y)$ is relatively large. Pursuing this goal is backed by information theory as well as standard estimation theory. On the estimation theory side, this property just amounts to ensuring that $Z$ be a sufficient statistic for $X$ and $Y$. It thus has importance in finding optimal estimators, for example, through the Rao-Blackwell theorem \cite{blackwell1947conditional}. On the information theoretic side, if $I(Y;Z) = H(Y)$, then having an instance $z$ would completely determine the corresponding instance $y$, and so there exists an estimator of $Y$ that takes $Z$ as input and has zero probability of error. This notion can be expanded to $I(Y;Z) < H(Y)$ by Fano's inequality  and its generalizations \cite{cover2012elements} \cite{verdu1994generalizing}. Fano's inequality provides the following bound on estimation error for \textit{any} estimator of $Y$ defined as a function of $Z$:
\begin{equation}\label{eqn:fano}
	h_2(P_e) + P_elog_2\left(|\mathcal{Y}| - 1 \right) \geq H(Y)-I(Y;Z)
\end{equation}
where $P_e$ is the error rate of the estimator and $h_2$ denotes the binary entropy function ${h_2(t)=-tlog_2(t)-(1-t)log_2(1-t)}$. This inequality has a left hand side (LHS) that is strictly increasing in $P_e$ for $P_e \leq \frac{1}{2}$. Thus the restriction of the LHS to $[0, \frac{1}{2}]$ is invertible, and since $H(Y)$ is fixed, we can say that $P_e$ is lower bounded by a monotonically decreasing function of $I(Y;Z)$. In some cases we do achieve near equality in (\ref{eqn:fano}) - particularly when 1.) the estimator performs (nearly) equally well on each class and 2.) the estimator ${Z \to \hat{Y}}$ incurs relatively low levels of compression when compared to that which was incurred in the map ${X \to Z}$.

\subsection{Information Losses}
We now turn to the reason for caring about the second aspect of finding a minimally sufficient statistic - the minimality. This is where the role of our sampled data comes into play, and with it, the concept of information losses. 

When we train on a finite sample of data, achieving the first aspect of a minimally sufficient statistic - the sufficiency - becomes difficult. This is because, no matter what representation we choose, we always have an information loss of the form:
\begin{equation}
    I_{Loss}^{(1)} \triangleq |I(Y;Z) - \hat{I}(Y;Z)|
\end{equation}
(The superscript (1) here is to distinguish between this form of information loss and another form which will appear later. We will call the current form \textit{type one information losses}). In choosing our representation, we will only be able to control the latter term in this expression, as that term corresponds to the model we have estimated from our training data. Thus, if this loss is large, then, no matter what we do, we will have trouble in making $I(Y;Z)$ as large as possible. 

Throughout this paper, we will find that this term, $I_{Loss}^{(1)}$, depends on $I(X;Z)$. In the old bounds (i.e. previous to this paper), its dependence is exponential \cite{shamir2010learning}:
\begin{equation}\label{eqn:ilb}
I_{Loss}^{(1)}  \leq \mathcal{O}\left(\sqrt{\frac{|\mathcal{Y}|}{2m}}2^{I(X;Z)}\right)
\end{equation}
where $m$ is the number of training samples. And so we see that, at least in this form, keeping $I(X;Z)$ low is pertinent. 

In this paper, we will find that the dependence on $I(X;Z)$ is relaxed to a linear one. Thus it may not always be so clear that we should minimize $I(X;Z)$. A perhaps more illuminating perspective can be found if we transfer instead to what we call \textit{type two information losses}. These relate the best possible representation (in terms of achieving sufficiency) to the one that we would obtain by optimizing $Z$ jointly with our estimated probability distribution. Before describing this new type of information loss, we will need to rigorously define the representations that we qualitatively described in the previous sentence.

\begin{defn}[]
Let $\epsilon>0$. We denote as $Z_{\epsilon}^*(I)$ and $\hat{Z}_{\epsilon}(I)$ any random variables that are at most ${\epsilon \text{-suboptimal}}$ for the following information bottleneck problems respectively: 
\begin{align*}
\underset{p(z|x)}{sup}&{I(Y;Z)}  \\
\text{subject to }& I(X;Z) = I
\end{align*}
\begin{align*}
\underset{p(z|x)}{sup}&{\hat{I}(Y;Z)}  \\
\text{subject to }& I(X;Z) = I
\end{align*}
\end{defn}
We will then define type two information losses as 
\begin{equation}
    I_{Loss, \epsilon}^{(2)}(I) \triangleq I(Y;Z_{\epsilon}^*(I)) - I(Y;\hat{Z}_{\epsilon}(I))
\end{equation}
which is, in general, a function of ${I \triangleq I(X;Z)}$. Then, rearranging, we see that the quantity we care about, $I(Y;\hat{Z}_{\epsilon}(I))$, is given by ${I(Y;Z_{\epsilon}^*(I)) - I_{Loss, \epsilon}^{(2)}(I)}$, and so picking an $I(X;Z)$ that maximizes this expression is critical, though it may not always result in a direct minimization of $I(X;Z)$. 

In any case, it is easy to convert bounds on type one information losses into corresponding bounds on type two information losses, as we will see in the next lemma.
\begin{lemma}[]\label{lemma:type_2}
Suppose that we have a bound of the form ${I_{Loss}^{(1)} \leq K(\cdot)}$, where $K(\cdot)$ can be any function of any number of arguments. Then: 
\begin{equation}
I_{Loss, \epsilon}^{(2)}(I) \leq 2K(\cdot) + \epsilon
\end{equation}
\label{IBBound}
\end{lemma}

\subsection{Automatic Implementation via Neural Networks}
There is evidence \cite{shwartz2017opening}\cite{achille2017emergence} that neural networks automatically solve the information bottleneck problem. The first set of evidence is experimental.  Authors of \cite{shwartz2017opening} found that a wide range of neural networks undergo training in two phases. In the first phase, the neural networks memorized the inputs. This corresponded to an increase of $I(X;Z)$ and $I(Y;Z)$ simultaneously. During this phase, the average magnitude of back-propagated gradients surpassed the variance. In the second phase, this dynamic swapped and the variance surpassed the average. During this phase, $I(Y;Z)$ increased, but $I(X;Z)$ dropped - the neural networks were compressing the input to learn more about $Y$. 

The second set of evidence is theoretical. The authors of \cite{achille2017emergence} show that $I(X;Z)$ is tightly related to the information between the weights and the data $I(W; \mc{D}^l)$. This relationship holds with only a few assumptions on the corresponding neural network. They then shown that $I(W; \mc{D}^l)$ is small when the network converges to a \textit{wide} local minimum of the cross entropy loss function. Finally, they argue that stochastic gradient descent tends to converge to such minima.

Some more recent experimental evidence \cite{michael2018on} counters these two arguments. This new evidence shows that some networks can achieve high $I(Y;Z)$  without compression. Thus some  networks can significantly outperform the lower bound of inequality (\ref{eqn:ilb}). This paper presents new lower bounds which are much tighter and less sensitive to $I(X;Z)$ than (\ref{eqn:ilb}). These bounds - while useful on their own right- help to explain this counter evidence.

\section{New Bounds on Information Losses} \label{sec:bounds}
We will now move on to deriving the new bounds on information losses.

\subsection{Product Form Decomposition - Intuition and Setup}
Our first major step is a decomposition of information losses into a product of two terms, one being $I(X;Z)$, and the other being a term related to a statistical distance between $\mathbb{P}$ and $\hat{\mathbb{P}}$. The proof of this decomposition takes some setting up. The setup is performed by generalizing the well studied maximal coupling \cite{sason2013entropy} from statistics to our purposes. We will call our generalization the \textit{conditional maximal coupling}, and will begin its construction by quickly reviewing couplings in general \cite{den2012probability}. 
\begin{defn}[Coupling]
Given two probability models $\mathbb{P}_{\tilde{S}}$ and $\mathbb{Q}_S$ on a list of variables $S$, a \textbf{coupling} of these models is a pair of random variables $(\tilde{S}, \hat{S})$ with joint distribution $\gamma_{\tilde{S}, \hat{S}}$ such that the marginal distributions satisfy $\gamma_{\tilde{S}} = \mathbb{P}_{\tilde{S}}$ and $\gamma_{\hat{S}} = \mathbb{Q}_S$.    
\end{defn}

\begin{const}[Conditional Maximal Coupling]\label{maximal_coupling}
We set our coupling $\left((\tilde{X}, \tilde{Y}, \tilde{Z}), (\hat{X}, \hat{Y}, \hat{Z})\right)$ as follows. First, define the function ${m_l: \mc{X} \times \mc{Y} \to [0,1]}$ through
\begin{equation}
m_l(a,b) := \text{min}\{p_{Y|X}(b|a), \hat{p}_{Y|X}(b|a) \}
\end{equation}
Next, define a real number $\rho$ as
\begin{equation}
\rho := \int \left(\sum_{y}m_l(x,y)\right) d\mathbb{P}_X
\end{equation}
and define $J$ as a Bernoulli random variable with success probability $\rho$. Then define variables ${U=(U_1, U_2), V=(V_1, V_2)}$ and ${W=({W_1, W_2})}$ through 
\begin{align}
p_{U_1,U_2}(u_1,u_2) &:= \frac{p_X(u_1)m_l(u_1, u_2)}{\rho} \\
p_{V_1,V_2}(v_1,v_2) &:= \frac{p_{X,Y}(v_1, v_2)- p_{X}(v_1)m_l(v_1, v_2)}{1- \rho}   \\
p_{W_1,W_2}(w_1,w_2) &:= \frac{\hat{p}_{X,Y}(v_1, v_2) - p_X(w_1)m_l(w_1, w_2)}{1- \rho}
\end{align}
Next define $(\tilde{X}, \tilde{Y}, \hat{X}, \hat{Y})$ as functions of the above random variables as follows:
\begin{equation}
\begin{cases}
(\tilde{X}, \tilde{Y}) = (\hat{X}, \hat{Y}) = (U_1, U_2)  & \text{if } J = 1 \\
(\tilde{X}, \tilde{Y}) = (V_1, V_2), ~(\hat{X}, \hat{Y}) = (W_1, W_2), & \text{if } J = 0
\end{cases}
\end{equation}
Finally, we define $\tilde{Z}$ and $\hat{Z}$ through 
\begin{equation}
\gamma_{\hat{Z}|\hat{X}} = \gamma_{\tilde{Z}|\tilde{X}} = p_{Z|X}
\end{equation}
\end{const}
\begin{lemma}\label{lemma:construction_correctness}
Construction \ref{maximal_coupling} yields a valid coupling. 
\end{lemma}
\begin{lemma}[]\label{lemma:max_couple_variation}
The definitions of Construction \ref{maximal_coupling} satisfy the following relationship:
\begin{equation}
1 - \rho = \gamma(\tilde{Y} = \hat{Y} | \tilde{X} = \hat{X}) =  \mathbb{E}_{\mathbb{P}_X}\left[\frac{1}{2}\sum_y \left|p(y|x)-\hat{p}(y|x)\right|\right]
\end{equation}
\end{lemma}

Motivated by Lemma \ref{lemma:max_couple_variation}, we will denote $1 - \rho$ as $\bar{\delta}(\hat{\mathbb{P}})$. This notation emphasizes its role as an average total variation distance. This finishes our setup for the decomposition, which we will now move on to prove. 

\subsection{Product Form Decomposition - Theorem and Proof}
\begin{theorem}[] \label{thm:estimated_info_bound}
\begin{align}
\left|I(Y;Z) - \hat{I}(Y;Z) \right| \leq \bar{\delta}(\hat{\mathbb{P}})I(X;Z) + h_2\left(\bar{\delta}(\hat{\mathbb{P}})\right)
\end{align}
\end{theorem}
\begin{proof} 

We will use several Markov chains in this proof. All of them follow from the following Bayesian network describing the generative process of all relevant random variables which is shown in figure \ref{bayes-net}.  Each Markov chain that we use comes from the fact that the $X$ variables d-separate the $Z$ variables from the rest of the network.    

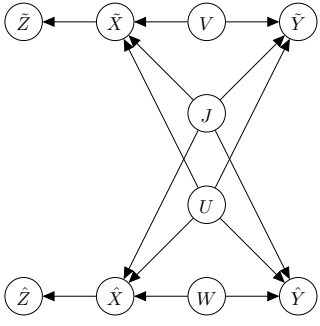
\begin{figure}
\begin{center}
\resizebox{!}{0.17\textheight}{
  \begin{tikzpicture}[text height=2ex]{
  	\node[latent] (j) {$J$}; %
    \node[latent, below=of j] (u) {$U$} ; %
    \node[latent, below=of u] (w) {$W$} ; %
    \node[latent, above=of j] (v) {$V$} ; %
    \node[latent, right=of v] (y1) {$\tilde{Y}$} ; %
    \node[latent, right=of w] (y2) {$\hat{Y}$} ; %
    \node[latent, left=of v] (x1) {$\tilde{X}$} ; %
    \node[latent, left=of w] (x2) {$\hat{X}$} ; %
    \node[latent, left=of x1] (z1) {$\tilde{Z}$} ; %
    \node[latent, left=of x2] (z2) {$\hat{Z}$} ; %

    \edge {x1} {z1} ; %
    \edge {x2} {z2} ; %
    \edge {j, u} {x1, x2, y1, y2} ; %
    \edge {v} {x1,y1} ; %
    \edge {w} {x2,y2} ; %
  }
  \end{tikzpicture}
  }
\end{center}
\caption{Bayesian network describing the relationships between random variables in the proof of Theorem \ref{thm:estimated_info_bound}.}
\label{bayes-net}
\end{figure}

First, via coupling, we have
\begin{align}
\left|I(Y;Z) - \hat{I}(Y;Z)\right| = \left|I(\tilde{Y};\tilde{Z}) - I(\hat{Y};\hat{Z})\right| 
\end{align}
We decompose the above terms as follows: 
\begin{align}
I(\tilde{Y};\tilde{Z}) = I(\tilde{Y};\tilde{Z}|\tilde{X}) + I(\tilde{X}; \tilde{Z}) - I(\tilde{X}; \tilde{Z}|\tilde{Y}) \\
I(\hat{Y};\hat{Z}) = I(\hat{Y};\hat{Z}|\hat{X}) + I(\hat{X}; \hat{Z}) - I(\hat{X}; \hat{Z}|\hat{Y}) 
\end{align}
But, due to the Markov chains ${\tilde{Z}-\tilde{X}-\tilde{Y}}$ and ${\hat{Z}-\hat{X}-\hat{Y}}$, we have ${I(\tilde{Y};\tilde{Z}|\tilde{X}) = I(\hat{Y};\hat{Z}|\hat{X}) = 0}$. Furthermore, ${I(\tilde{X}; \tilde{Z}) = I(\hat{X}; \hat{Z}) = I(X;Z)}$, so:
\begin{equation}
\left|I(\tilde{Y};\tilde{Z}) - I(\hat{Y};\hat{Z})\right| = \left|I(\hat{X};\hat{Z}|\hat{Y}) - I(\tilde{X};\tilde{Z}|\tilde{Y}) \right| 
\end{equation}
We can further decompose each of these terms as: 
\begin{align}
I(\hat{X};\hat{Z}|\hat{Y}) = I(\hat{Z};\hat{X}|J,\hat{Y}) + I(\hat{Z};J|\hat{Y}) - I(\hat{Z};J|\hat{X},\hat{Y}) \notag \\
I(\tilde{X};\tilde{Z}|\tilde{Y}) = I(\tilde{Z};\tilde{X}|J,\tilde{Y}) + I(\tilde{Z};J|\tilde{Y}) - I(\tilde{Z};J|\tilde{X},\tilde{Y})
\end{align}
But we have from the Markov chains ${\hat{Z}-\hat{X}-J}$ and ${\tilde{Z}-\tilde{X}-J}$ that ${I(\hat{Z};J|\hat{X},\hat{Y}) = I(\tilde{Z};J|\tilde{X},\tilde{Y}) = 0}$, so these terms will disappear from the decomposition. Next, we can break down the term $I(\hat{Z};\hat{X}|J,\hat{Y})$ to: 
\begin{align}
&\rho I(\hat{Z};\hat{X}|J=1, \hat{Y}) + (1-\rho)I(\hat{Z};\hat{X}|J=0, \hat{Y}) \notag \\
= &\rho I(\hat{Z};U_1|U_2) + \bar{\delta}(\hat{\mathbb{P}})I(\hat{Z};W_1|W_2)
\end{align}
and similarly, we can break down:
\begin{align}
I(\tilde{Z};\tilde{X}|J,\tilde{Y}) = \rho I(\tilde{Z};U_1|U_2) + \bar{\delta}(\hat{\mathbb{P}})I(\tilde{Z};V_1|V_2)
\end{align}
But when ${\tilde{X} = \hat{X} = U_1}$, ${I(\hat{Z};U_1|U_2) = I(\tilde{Z};U_1|U_2)}$.
Thus, in total, ${\left|I(Y;Z) - \hat{I}(Y;Z)\right|}$ is given by:
\begin{align}
&\left|\bar{\delta}(\hat{\mathbb{P}})\left(I(\hat{Z};W_1|W_2)  - I(\tilde{Z};V_1|V_2)\right) + I(\hat{Z};J|\hat{Y}) - I(\tilde{Z};J|\tilde{Y}) \right| \notag \\
\end{align}
which can be bounded by the triangle inequality on each inner term.

Now, from the Markov chains ${\hat{Z}-\hat{X}-W_1}$, ${\hat{Z}-\hat{X}-W_2}$, ${\tilde{Z}-\tilde{X}-V_1}$, and ${\tilde{Z}-\tilde{X}-V_2}$, we have (via applications of the data processing inequality and its corollaries \cite{cover2012elements}): 
\begin{align}
I(\hat{Z};W_1|W_2) \leq  I(\hat{Z}; \hat{X}|W_2) \leq I(\hat{Z};\hat{X}) = I(X;Z) \\
I(\tilde{Z};V_1|V_2) \leq  I(\tilde{Z}; \tilde{X}|V_2) \leq I(\tilde{Z};\tilde{X}) = I(X;Z)
\end{align}
Further, $I(\hat{Z};J|\hat{Y}) \leq H(J)$ and $I(\tilde{Z};J|\tilde{Y}) \leq H(J)$. 
Then as ${0\leq a \leq c ~\wedge~ 0 \leq b \leq c \implies |a-b|\leq c}$, we have 
\begin{align}
\left|I(Z;W_1|W_2) - I(\tilde{Z};V_1|V_2)\right| \leq I(X;Z) \\
\left|I(\hat{Z};J|\hat{Y}) - I(\tilde{Z};J|\tilde{Y}) \right| \leq H(J) = h_2\left(\bar{\delta}(\hat{\mathbb{P}})\right)
\end{align}
And so, in total, we have
\begin{equation}
\left|I(Y;Z) - \hat{I}(Y;Z)\right| \leq \bar{\delta}(\hat{\mathbb{P}})I(X;Z) + h_2\left(\bar{\delta}(\hat{\mathbb{P}})\right)
\end{equation}
which completes the proof.
\end{proof}

A potentially useful special case of this bound occurs when we set $Z=X$:
\begin{cor}\label{cor_ZX_tight}
If $X$ is discrete, 
\begin{equation}
\left|I(X;Y) - \hat{I}(X;Y) \right| \leq \bar{\delta}(\hat{\mathbb{P}})H(X) + h_2(\bar{\delta}(\hat{\mathbb{P}}))
\end{equation}
\label{XBound}
\end{cor}
But we won't be using this corollary in the rest of the paper.

\subsection{Understanding $\bar{\delta}(\hat{\mathbb{P}})$}\label{subsec:understanding}
The above relationships \textit{looks} linear on $I(X;Z)$. However, $\hat{p}(y|x)$ is typically learned jointly with $Z$ and therefore $\bar{\delta}(\hat{\mathbb{P}})$ may itself depend on $I(X;Z)$. Thus we cannot yet say that this relationship is truly linear, and we certainly cannot yet say that it is tight.  Before we can make those claims, we will need to study $\bar{\delta}(\hat{\mathbb{P}})$ explicitly. We will begin with a `sanity-check' lemma. This lemma shows us that $\bar{\delta}(\hat{\mathbb{P}})$ does at least converge with the convergence of a typical neural classifier loss function. It arises from an application of Pinsker's inequality \cite{csiszar2011information}.
\begin{lemma}\label{lemma:sanity}
Suppose that $H(Y|X) = 0$. Then: 
\begin{equation}
\bar{\delta}(\hat{\mathbb{P}}) \leq \sqrt{\frac{1}{2}H_{\mathbb{P},\hat{\mathbb{P}}}(Y|X)}
\end{equation}
where $H_{\mathbb{P},\hat{\mathbb{P}}}(Y|X)$ is the conditional cross entropy between $\mathbb{P}$ and $\hat{\mathbb{P}}$, i.e. the usual cross entropy loss function.
\label{delta_CE}
\end{lemma}
This lemma is particularly applicable when we are estimating our cross entropy error on a validation set, as we can then take $\mathbb{P}$ in this lemma to be the empirical measure corresponding to the validation or training sample, in which we are almost certain to have ${H(Y|X)=0}$. In this sense Lemma \ref{lemma:sanity} can bound such empirical estimates of ${\bar{\delta}(\hat{\mathbb{P}})}$.

\subsection{Bounding $\bar{\delta}(\hat{\mathbb{P}})$ - Setting}
Finally, we will derive a rate of decrease for $\bar{\delta}(\hat{\mathbb{P}})$ in a general continuous learning algorithm. Our setup will involve defining a learning algorithm as a continuous map from a special topology on input probability measures on $\mc{X}\times \mc{Y}$ to conditional probability functions. This is basically to say that, given a training dataset (i.e. an empirical measure on $\mc{X}\times\mc{Y}$), we have a well-behaved way of obtaining the corresponding $\hat{p}_{\nu}(y|x)$. This is just slightly generalized so that we can consider any input measure (empirical or not) as a `training dataset'. We begin by reviewing that special topology, and then we will construct the topology that we will place on our output conditional probability distributions.

\begin{defn}
Let $M_1$ denote the set of Borel probability measures on $\mc{X}\times \mc{Y}$. Then the $\tau$-topology \cite{dembo1998large} (page 263) is the topology generated by the sets ${W_{f,r,c} = \{\nu: |\int f d\nu - r| < c\}}$ for all bounded Borel measurable functions ${f: \mc{X} \times \mc{Y} \to \mathbb{R}}$, all $r \in \mathbb{R}$ and all $c>0$. If we restrict $f$ to bounded continuous functions, we get the weak topology $\mc{W}$, which is strictly coarser than the $\tau$-topology.
\end{defn}

\begin{defn}
Let $\Sigma_{|\mc{Y}|}$ be the probability simplex in $|\mc{Y}|$ dimensions. Let $L^1(\mc{X})$ denote the space of absolutely integrable functions from $\mc{X}$ to $\mathbb{R}$ with norm ${\|f\|_{L_1} = \int |f|d\mathbb{P_X}}$. Let $L^1(\mc{X})^{|\mc{Y}|}$ denote the product space on $L^1(\mc{X})$, consisting of functions from $\mc{X}$ to $\mathbb{R}^{|\mc{Y}|}$ which are absolutely integrable in each output dimension, and with norm ${\|f\|_{L^{|\mc{Y}|}_1} = \frac{1}{2}\int \sum_y|f(x,y)|d\mathbb{P_X}}$. Finally, let $L^1(\mc{X}, \Sigma_{|\mc{Y}|})$ denote the subspace of $L^1(\mc{X})^{|\mc{Y}|}$ to the set of functions whose co-domain is  $\Sigma_{|\mc{Y}|}$.  
\end{defn}

The topology we've placed on $L^1(\mc{X}, \Sigma_{|\mc{Y}|})$ is metrized by the conditional total variation function that we've been working with. 
With these topologies defined, we will restrict ourselves to the study of algorithms which act as continuous maps between these topologies. This essentially requires that, when our training datasets are very similar (e.g. moving one training point to a point within a distance $\epsilon$ from the original), our algorithm will return very similar output functions in terms of conditional total variation. Thus this condition is somewhat related to algorithmic stability \cite{hardt2015train}, though not completely equivalent.

We will obtain two bounds on $\bar{\delta}(\hat{\mathbb{P}})$ in the remains of this paper. The first is asymptotic, and applies when we have continuity from the $\tau$-topology. The second is non asymptotic, and applies when we further have continuity from the weak topology. We will next show that gradient descent algorithms, under mild conditions, achieve these continuities. 

\begin{theorem}
Let $\Theta$ denote a normed parameter space and let ${\mc{L}: \mc{X}\times\mc{Y}\times\Theta \to \mathbb{R}}$ denote a loss function  which is integrable in $\mc{X}\times \mc{Y}$ for each $\theta \in \Theta$, which is differentiable with respect to $\theta$ for all ${(x,y)\in \mc{X}\times\mc{Y}}$, and whose $\theta$-gradients yield bounded continuous functions on $\mc{X}\times\mc{Y}$ when evaluated at each point ${\theta \in \Theta}$. Suppose further that our parameter space admits lipschitz-continuous outputs for each $(x,y)$. That is, ${|p_{\theta_1}(y|x) - p_{\theta_2}(y|x)| < L\|\theta_1 - \theta_2\|~ \forall (x,y)\in \mc{X} \times \mc{Y}}$. Then gradient descent applied to the empirical risk minimization of $\mc{L}$, with a fixed initiation $\theta^{(0)}$ and which proceeds for a fixed number of iterations, is continuous from $(M_1, \mc{W})$ to $L^1(\mc{X}, \Sigma_{|\mc{Y}|})$. 

If we relax the condition that the $\theta$ gradients of $\mc{L}$ be bounded continuous functions on $\mc{X}\times\mc{Y}$ when evaluated at each point ${\theta \in \Theta}$ to just bounded measurable functions, then this algorithm is still continuous from $(M_1, \tau)$ to $L^1(\mc{X}, \Sigma_{|\mc{Y}|})$.
\end{theorem}

\begin{proof}
The assumptions on $\mc{L}$ allow us to differentiate (with respect to $\theta$) under the integral sign. Let $\alpha_k$ denote the step size of the $k^{th}$ iteration. Let $\nu^* \in M_1$. We proceed by induction on the number of iterations.  

Let $\epsilon > 0$.  Let ${\delta_1=\frac{2\epsilon}{L\alpha_1|\mc{Y}|}}$. Let $\nu^* \in M_1$ and let $\nu$ be contained in the open set of the weak topology given by ${\{\nu: |\int \nabla_{\theta^{(0)}}d\nu - \int \nabla_{\theta^{(0)}}d\nu^*| < \delta_1 \}}$ (which clearly contains $\nu^*$). Let $\theta_*^{(1)}$ denote the parameter chosen after one gradient update when training on $\nu^*$, and let $\theta^{(1)}$ denote the parameter chosen after one gradient update when training on $\nu$. Then:
\begin{align}
    \|\theta_*^{(1)} -  \theta^{(1)}\| = \left \|\alpha_1\left(\int \nabla_{\theta^{(0)}}d\nu - \int \nabla_{\theta^{(0)}}d\nu^*\right) \right \| \leq \alpha_1\delta_1
\end{align} 
so
\begin{align}
    \frac{1}{2}\int \sum_y \|p_{\theta_*^{(1)}}(y|x) -  p_{\theta^{(1)}}(y|x)\|d\mathbb{P}_X \leq \frac{L|\mc{Y}|\alpha_1\delta_1}{2} = \epsilon
\end{align}
and so the hypothesis is true if our algorithm consists of one iteration. 

Suppose that the hypothesis when we use $(k-1)$ iterations. Let $\epsilon>0$. Let ${\delta_{k-1} = \frac{\epsilon}{L|\mc{Y}|}}$ and let ${\delta_{k}=  \frac{\epsilon}{L|\mc{Y}|\alpha_k}}$. Chose an open set $U$ of the weak topology such that ${\|\theta_*^{(k-1)} - \theta_c^{(k-1)} \| \leq \delta_k}$ when ${\nu_c \in U}$ which is possible by the induction hypothesis, and where $\theta_*^{(k-1)}$ and $\theta_c^{(k-1)}$ denote the chosen parameters after iteration ${k-1}$ of the gradient descent when trained on $\nu^*$ and $\nu_c$. Let ${\nu \in U \cap \{\nu: |\int \nabla_{\theta^{(k-1)}}d\nu - \int \nabla_{\theta^{(k-1)}}d\nu^*| < \delta_k \}}$. Then by the triangle inequality:
\begin{align}
    \|\theta_*^{(k)} -  \theta^{(k)}\| \leq \delta_{k-1} + \alpha_k\delta_k
\end{align} 
so the conditional total variation between $p_{\theta_*^{(k)}}(y|x)$ and $p_{\theta^{(k)}}(y|x)$ is less than or equal to ${\frac{L|\mc{Y}|(\delta_{k-1} + \alpha_k\delta_k)}{2}}$ which is equal to $\epsilon$.

For the final statement, note that all of the above open sets in the $\mc{W}$-topology used in this proof remain open sets in the $\tau$-topology when we relax the conditions of $\mc{L}$. This completes the proof.
\end{proof}

\subsection{Bounding $\bar{\delta}(\hat{\mathbb{P}})$ - The Asymptotic Case}
We now wish to bound the conditional total variation of an estimated model against the true model when we use such a general learning algorithm in our setting. We will re-label $\bar{\delta}(\hat{\mathbb{P}})$ to $\bar{\delta}(\mathbb{P}_f)$ to emphasize that our estimated model is coming from such an algorithm. We then have the following asymptotic theorem on the rate of decay for $\bar{\delta}(\mathbb{P}_f)$. This will apply whenever we have continuity from the $\tau$-topology in our algorithm, and will be used in our non-asymptotic specialization that follows. We will use two final lemmas in both of those proofs.

\begin{lemma}\label{lemma:kl_change}
Let $(\Omega, \mc{F}, \mu)$ be a probability space and let ${h: \Omega \to \mathbb{R}}$ be bounded and measurable. Let $\mc{G}$ denote the set of non-negative measurable functions with expectation $1$. Then ${\underset{g\in\mc{G}}{\inf~}\mathbb{E}\left[g\cdot(h+log ~g) \right] = -log~\mathbb{E}\left[e^{-h(\omega)} \right]}$.
\end{lemma}

\begin{lemma}\label{lemma:ijensen}
Let $(\Omega, \mc{F}, \mu)$ be a probability space and let ${f: \Omega \to \mathbb{R}}$ be bounded and measurable with ${Range(f) \subseteq [0,1]}$. Then ${log~\left(\mathbb{E}\left[e^{-2f^2}\right]\right) \leq -2\mathbb{E}\left[f\right]^2}$.  
\end{lemma}

\begin{theorem}[]\label{thm:rate_asy}
Let ${\epsilon \in (0, 1)}$, and let ${0 < \zeta < 1}$. If $\mc{F}$ is a continuous learning algorithm from $(M_1, \tau)$ to ${L^1(\mc{X}, \Sigma_{|\mc{Y}|})}$ such that, for any $\nu \in M_1$, the total variation between $\mc{F}\nu$ and $\nu_{y|x}$ is smaller than the total variation between $\mc{F}\nu$ and $p_{y|x}$ at any point in the support of $\nu$. Suppose further that the `training' total variation, ${\mathbb{E}_{\nu}\left[\frac{1}{2}\sum_y |\nu_{y|x} - \mc{F}\nu|\right]}$, is bounded above by $\zeta$. Then:
\begin{align}
\underset{m \to \infty}{\text{limsup }} &\frac{1}{m} log~\mathbb{P}^{m}(\bar{\delta}(\mathbb{P}_f) \geq \epsilon) \leq 4\zeta - 2\epsilon^2
\end{align}
where ${\mathbb{P}^m}$ is the probability measure on $M_1$ induced by the sampling of $m$ data-points on $\mc{X}\times\mc{Y}$.
\end{theorem}

\begin{proof}  
For notational convenience, we will denote as $\delta_{\nu}(x)$ the conditional total variation between $p(y|x)$ and $(\mc{F}\nu)(y|x)$ for a fixed $x$.

We will first need to show that the map ${\bar{\delta}: M_1 \to \mathbb{R}}$, given by ${\nu \mapsto \mathbb{E}_{\mathbb{P}_X}\left[\delta_{\nu} \right]}$ is continuous from the $\tau$-topology to the Euclidean topology. This is trivial since $\mathbb{E}_{\mathbb{P}_X}\left[\delta_{\nu} \right]$ is just the composition of $\mc{F}$, which was assumed continuous, with the fixed-point distance function $d(\cdot, p_{y|x}(y|x))$ defined over ${L^1(\mc{X}, \Sigma_{|\mc{Y}|})}$. 

Now, let ${\Gamma = \{\nu \in M_1 : \mathbb{E}_{\mathbb{P}_X}\left[\delta_{\nu}\right] \geq \epsilon\}}$. By the above continuity and by the fact that $[\epsilon, 1]$ is closed in $\mathbb{R}$, we have that $\Gamma$ is closed. Then, by Sanov's Theorem \cite{dembo1998large}:
\begin{align}
\underset{m \to \infty}{\text{limsup }} \frac{1}{m}log~\mathbb{P}^m(\mathbb{P}_{f} \in \Gamma) \leq -\underset{\nu \in \Gamma}{\text{inf }}\mc{D}_{KL}(\nu || p(x,y))
\end{align}
We thus wish to lower bound $\mc{D}_{KL}(\nu || p(x,y))$ over $\Gamma$. We begin by decomposing $\frac{d\nu}{d\mathbb{P}}$ into ${\frac{d\nu_x}{d\mathbb{P}_x}\frac{\nu_{y|x}}{p_{y|x}}}$. Where $\nu_x$ and $\mathbb{P}_x$ are the marginal distributions of $\nu$ and $p(x,y)$ on $\mc{X}$. We are guaranteed that the functions and $\nu_{y|x}$ exist on the support of $\nu_x$ since $y$ is discrete. The KL-divergence then becomes: ${\mc{D}_{KL}(\nu || p(x,y)) = \mathbb{E}_{\mathbb{P}_X}\left[\frac{d\nu_x}{d\mathbb{P}_x}(\tilde{h} + log \frac{d\nu_x}{d\mathbb{P}_x})\right]}$ where ${\tilde{h} \triangleq \sum_y \nu_{y|x}log~\frac{\nu_{y|x}}{p_{y|x}}}$ is bounded below (via Pinsker's inequality) by the function $2\left(\sum_y|p_{y|x} - \nu_{y|x}|\right)^2$, which itself is bounded below by $2\left(\sum_y|p_{y|x} - \mc{F}\nu| - \sum_y|\nu_{y|x} - \mc{F}\nu|\right)^2$ because the absolute value of the second term in this expression is smaller than that of the first term for each point in the support of $\nu$. The first term is just the function $\delta_{\nu}$ defined at the start of this proof. We will call the second term ${\delta_{\nu}^t}$. We can lower bound this expression one more time with ${2\delta_{\nu}^2 - 4\delta_{\nu}^t}$. We are left with:
\begin{equation}
    \mc{D}_{KL}(\nu || p(x,y)) \geq \mathbb{E}_{\mathbb{P}_X}\left[\frac{d\nu_x}{d\mathbb{P}_x}(2\delta_{\nu}^2 + log \frac{d\nu_x}{d\mathbb{P}_x})\right] - 4\mathbb{E}_{\nu}\left[\delta_{\nu}^t \right]
\end{equation}

We will bound these two remaining terms separately. The second is taken care of in this theorem's hypothesis, being bounded below by $-4\zeta$. For the latter, we can combine Lemmas \ref{lemma:kl_change} and \ref{lemma:ijensen} to obtain a lower bound of $2\epsilon^2$ (since $\nu\in \Gamma$).

Since neither of these two bounds depend on $\nu$, negating their sum yields the result. 
\end{proof}

\subsection{Bounding $\bar{\delta}(\hat{\mathbb{P}})$ - The Non-Asymptotic Case}
The previous theorem gives us: 
\begin{equation} \label{full_delta_bound}
\mathbb{P}^m(\bar{\delta}(\mathbb{P}_{f}) \geq \epsilon) \leq e^{m(4\zeta-2\epsilon^2) + o(m)}
\end{equation}
where $o(m)$ refers to any terms such that ${\underset{m \to \infty}{\text{lim }}{\frac{o(m)}{m}} = 0}$. We will need to study $o(m)$ since it's somewhat of an unknown here, and may be large for small $m$. The next theorem, which is non-asymptotic, will take care of this when $\mc{F}$ is continuous from the weak topology.

\begin{theorem}\label{thm:non_asy_small}
Take all assumptions from Theorem \ref{thm:rate_asy}, but remove the assumption that $\mc{F}$ be a continuous map from $(M_1, \tau)$ to ${L^1(\mc{X}, \Sigma_{|\mc{Y}|})}$ and assume it is instead continuous linear from ${(M_1, \mc{W})}$. Suppose further that $\mc{X}$ is compact, and that $p(x)$ has full support with density ${p(x,y)> 0}$ everywhere. Then there exists a function ${k(m'): \mathbb{Z}^+ \to \mathbb{R}}$ with ${k(m') \leq \sqrt{m'}}$ such that:
\begin{align}\label{eqn:non-asy}
     \mathbb{P}^m(\bar{\delta}(\mathbb{P}_f) \geq \epsilon) \leq \underset{m'\in \mathbb{Z}^+}{\inf~} 2^{m'|\mc{Y}|}e^{-2m\left( \left(\epsilon-2\frac{k(m')}{\sqrt{m'}} \right)^2 -4\zeta\right) + 2\frac{k(m')}{\sqrt{m'}}}
\end{align}
(A more detailed description of $k(m')$, from which we can discover more of its properties, is contained in the proof).
\end{theorem}
\begin{proof}
    Let the notations $\delta_{\nu}$ and $\Gamma$ be defined as they were in the proof of Theorem $\ref{thm:rate_asy}$.
    
    Let $E(S_{m'}, k(m'))$ constitute a family of conditions, indexed first by samples of $m'$ points of $\mc{X}$ and second by functions ${\mathbb{Z}^+ \to \mathbb{R}}$, which constitute that ${|\mathbb{E}_{p(x)}\left[ \delta_{\nu}\right] - \mathbb{E}_{S_{m'}}\left[ \delta_{\nu} \right]| \leq \frac{k(m')}{\sqrt{m'}}}$, where the second expectation is the monte-carlo estimate over the indexed sample.
    
    Let the sets $\Gamma(S_{m'}, i)$, indexed first over samples of $\mc{X}$ consisting of $m'$ points and second over the set ${1, 2, \cdots, 2^{m'|\mc{Y}|}}$,  be given by ${\Gamma(S_{m'}, i) = \{h: \mathbb{E}_{p(x)}\left[\delta_{h} \right] \geq \epsilon, ~ \mc{F}h(y|x_j) \geq / \leq p_{y|x}(y|x_j) \}}$ (where the ${x_j's}$ run over the sampled points in ${S_{m'}}$ and $i$ runs over the possible choices of ${\geq / \leq}$). Let ${F(S_m', i, k(m'))}$ denote the family of conditions ${\{\nu: \mathbb{E}_{S_{m'}}\left[\delta_{\nu} \right] \geq \epsilon - \frac{k(m')}{\sqrt{m'}},  ~ \mc{F}\nu(y|x_j) \geq / \leq p_{y|x}(y|x_j) \}}$ where the $x_j$ run over the sampled points and the choices of $\geq$ and $\leq$ correspond to those of $\Gamma^i$. Let $G(S_{m'}, i)$ denote the condition on measures ${\mu \in M_1}$ such that there exists a measure ${\mu' \in \Gamma(S_{m'}, i)}$ with ${\mu'_{y|x} = \mu_{y|x}}$. Note that ${E(S_{m'}, k(m')) \cap G(S_{m'}, i) \subseteq F(S_m', i, k(m'))}$. 

    Let $M$ denote the vector space of finite signed measures on ${\mc{X}\times\mc{Y}}$ endowed with the weak topology. For any probability measure ${\nu'_x \in M_1(\mc{X})}$, let $R^{\nu'_x}$ be the subspace of measures with marginal distribution $\nu'_x$. Let ${R_1^{\nu'_x}}$ be the subset of $R^{\nu'_x}$ consisting of probability measures. Define a linear map on ${R_1^{\nu'_x}}$, denoted $\mc{C}_{\nu'_x}$, which takes $\nu'$ to its disintegration $\nu'_{y|x}$.  
    
    Let ${f_{\nu'_x}: M_1 \times \mc{C}_{\nu'_x}R_1^{\nu'_x}}$ denote the family of real valued function (indexed by $M_1(\mc{X})$) taking ${(\nu, \nu_{y|x}')}$ to the value ${\mathbb{E}_{\nu_x}\left[\sum_y \nu_{y|x} log\frac{\nu'_{y|x}}{p_{y|x}} + log~\frac{d\nu_x}{d\mathbb{P}_X}\right]}$, which is to be taken as infinite when the support of $\nu_x'$ is not a superset of the support of $\nu_x$, and is further infinite when $\nu_x$ is not absolutely continuous with respect to $p(x)$. Note that each ${f_{\nu'_x}(\cdot, a)}$ is convex and continuous in the weak topology for each fixed $a$ (as ${p(x)>0}$ and ${p_{y|x}>0}$ everywhere by the theorem's hypothesis), and each ${f_{\nu'_x}(b, \cdot)}$ is concave and continuous for each fixed $b$.
    
    Now, since ${\mc{X} \times \mc{Y}}$ is compact, $M_1$ is compact in the weak topology. Then for any $\nu'_x$, $R_1^{\nu'_x}$ is compact (being a closed subset of a compact space). Then ${\mc{C}_{\nu'_x}R_1^{\nu'_x}}$ is compact and convex. We also have that the subsets ${G(S_{m'}, i)}$, ${E(S_{m'}, k(m'))}$, and ${F(S_{m'}, i, k(m'))}$ are all closed, and therefore compact. We also have convexity in ${F(S_{m'}, i, k(m'))}$, but not in the other two. 
    
    Arbitrarily pick some ${\nu{''}_x \in M_1}$ with full support and denote $f$ as $f_{\nu''_x}$ as $f$. Let ${r(S_{m'}, i, k(m'))}$ denote the minimum of the expression ${f(a, a_{y|x})}$ over ${K(S_{m'}, i) \cap E(k(m')) \cap F(S_{m'}, i, k(m'))}$ and denote the minimizer as ${a(S_{m'}, i, k(m'))}$. The image of the map ${f(\cdot, a(S_{m'}, i, k(m')))}$ is a compact subset of $\mathbb{R}$ - i.e. a closed and bounded interval ${\mc{I}(S_{m'}, i, k(m'))}$. Let ${\tilde{\mc{I}}(S_{m'}, k(m'))}$ denote the union of these intervals over the finite indices $i$. Cover this interval with a family of subintervals ${\tilde{\mc{I}}(S_{m'}, k(m'), j)}$ of size $\frac{k(m')}{\sqrt{m'}}$. 
    
    We will now fix $k(m')$ to be the smallest number such that \textit{there exists} a sample $S^*_{m'}$ in which both ${G(S^*_{m'}, i) \cap E(S^*_{m'}, k(m')) \neq \emptyset}$ for all $i$ in which ${G(S^*_{m'}, i) \neq \emptyset}$ and ${\mc{I}(S^*_{m'}, k(m'), j) \cap E(S^*_{m'}, k(m')) \neq \emptyset}$ for all $j$ in which ${\tilde{\mc{I}}(S^*_{m'}, k(m'), j) \neq \emptyset}$. Such a $k(m')$ exists, and is less than or equal to $\sqrt{m'}$ since $E(S_{m'}, \sqrt{m'})$ is all of $M_1$. Fix $S_{m'}$ to any of the samples that we just established the existence of.  We will drop the notations $S_{m'}$ and $k(m')$ from the notation for any conditions referring to them from now on.  

    Now, denote as $C_b(\mc{X})$ the set of bounded continuous functions from $\mc{X}$ to $\mathbb{R}$ and construct a family of maps ${\mc{G}_{\lambda, \nu'}: M_1 \to \mathbb{R}}$ indexed over ${\lambda \in C_b(\mc{X})}$ and ${\nu' \in M_1}$ which takes ${\nu \in M_1}$ to ${\mathbb{E}_{\nu}\left[mlog\frac{\nu'_{y|x}}{p_{y|x}} + m\lambda \right]}$. Then for any empirical ${L_m \in \Gamma(i)}$ corresponding to a sample of $m$ points, we have that ${\mc{G}_{\lambda, \nu'}L_m \geq \underset{\nu \in \Gamma(i)}{\inf}\mc{G}_{\lambda, \nu'}\nu}$ for all $\lambda, \nu'$. Thus the probability that $L_m$ is in $\Gamma(i)$ is bounded above by the probability that ${\mc{G}_{\lambda, \nu'}L_m - \underset{\nu \in \Gamma(i)}{\inf}\mc{G}_{\lambda, \nu'}\nu \geq 0}$. Then by Chernoff's inequality, we have that ${\frac{1}{m}log~\mathbb{P}^m\left({L_m \in \Gamma(i)}\right)}$ is bounded above by:
    \begin{equation}\label{eqn:logp}
      \frac{1}{m}log~\mathbb{E}\left[e^{m\mathbb{E}_{L_m}\left[log\frac{\nu'_{y|x}}{p_{y|x}} + \lambda \right]}\right] - \underset{\nu \in \Gamma(i)}{\inf}\mathbb{E}_{\nu}\left[log\frac{\nu'_{y|x}}{p_{y|x}} +\lambda\right]
    \end{equation}
    where the first expectation is taken over $\mathbb{P}^m$.     
    
    The first term can be reduced to ${log~\mathbb{E}_{p(x)} \left[e^{\lambda} \right]}$. Optimizing over $\lambda$ yields a bound of
    \begin{equation}\label{eqn:b1}
        -\underset{\lambda}{\sup~}\underset{\nu \in \Gamma(i)}{\inf~}\mathbb{E}_{\nu}\left[log\frac{\nu'_{y|x}}{p_{y|x}}\right] + \mathbb{E}_{\nu}\left[\lambda\right] - log(\mathbb{E}_{p(x)}\left[e^{\lambda} \right])
    \end{equation}
     We will denote as $\Gamma^i_{y|x}$ the set of conditional probability functions ${\nu_{y|x}}$ such that there exists $\nu \in \Gamma(i)$ with disintegration given by $\nu_{y|x}$. We will also denote a function $g_{\nu'}(\nu_{y|x}, \mu_x)$ defined on ${\Gamma^i_{y|x} \times M_1(\mc{X})}$ which yields ${\mathbb{E}_{\mu_x\nu_{y|x}}\left[log\frac{\nu'_{y|x}}{p_{y|x}}\right]}$ when the support of the latter argument is equal to the domain of the former, and is infinite otherwise. Note that $g$ is convex and lower-semicontinuous in $\mu_x$ for fixed $\nu_{y|x}$ since it is linear in the convex subset ${\{\mu_x \in M_1(\mc{X}) : supp(\mu_x) = Dom(\nu_{y|x})\}}$ and infinite outside of this subset. Finally, we will define the function ${h: M_1(\mc{X}) \times C_b(\mc{X}) \to \mathbb{R}}$ given by ${h(\mu_x, \lambda) = \mathbb{E}_{\mu_x}\left[\lambda\right] - log(\mathbb{E}_{p(x)}\left[e^{\lambda} \right])}$. This function is concave in $\lambda$, convex in $\mu_x$, and lower semicontinuous in $\mu_x$ \cite{dembo1998large}. Then (\ref{eqn:b1}) is upper bounded by:
    \begin{equation}
        -\underset{\lambda \in C_b}{\sup} \underset{\nu_{y|x}\in\Gamma^i_{y|x}}{\inf}\underset{\mu_{x}\in M_1(\mc{X})}{\inf} g_{\nu'}(\nu_{y|x}, \mu_x) + h(\mu_x, \lambda)
    \end{equation}
    
    Note also that the the objective function of this expression is decoupled for $\nu_{y|x}$ and $\lambda$. We can thus swap the supremum with the first infinum. But then inside the first infinum, we are left with an objective function in which a minimax theorem applies \cite{sion1958general} because $M_1(\mc{X})$ is compact and convex in the weak topology when $\mc{X}$ is compact, and so we can swap the supremum with the second infinum as well. Since the first term does not depend on $\lambda$, we can then consider for each fixed $\mu_x$ the expression ${\underset{\lambda}{\sup~} h(\mu_x, \lambda)}$. But the supremum of this function over ${\lambda \in C_b(\mc{X})}$ is none other than the $KL$ divergence between $\mu_x$ and $p(x)$ \cite{donsker1975asymptotic}.  We are thus left with a full upper bound of (now optimizing  over ${\nu_{y|x}' \in \mc{C}_{\nu''_x}R_1^{\nu''_x}}$):
    \begin{align}\label{eqn:pre_inf}
        -\underset{\nu' }{\sup}\underset{\nu \in G(i)} f(\nu_{y|x}, \nu'_{y|x})
    \end{align}
    We would be able to swap the supremum and infinum if our feasible set were convex and compact. This is true for our search space over $\nu'$, but not for $G(i)$. Our goal is to then transform $G(i)$ into $F(i)$, which is convex, with corresponding error terms included. This can be done by tightening $G(i)$ to ${G(i)\cap E}$ and then relaxing that set to $F(i)$, this will incur some error, but if we end up choosing ${\nu_{y|x}'}$ to be the disintegration of $a(i)$, then this error will be bounded by $\frac{k(m')}{\sqrt{m'}}$. 

    With our feasible set now being $F(i)$, we can swap the supremum and infinum, and then pick $\nu'_{y|x}$ to be equal to $\nu_{y|x}$ on the support of $\nu$, and arbitrary elsewhere. The objective function is then just the minimum $KL$ divergence over $F(i)$, which we know how to deal with due to the proof of Theorem \ref{thm:rate_asy}. Minimizing then gives us ${\nu_{y|x}=\nu'_{y|x}}$ both given by the disintegration of $a(i)$, and with the objective function bounded by ${\underset{\nu \in F(i)}{\inf}2\mathbb{E}_{p(x)}\left[\delta_{\nu} \right]^2 -4\zeta}$. If we again add the constraint $E$ to the feasible region (with another error of at most $\frac{k(m')}{\sqrt{m'}}$ added on), then this is bounded above by ${2(\epsilon - 2\frac{k(m')}{\sqrt{m'}})^2}$. Union bounding over $i$ yields the result.
\end{proof}

\subsection{Some Insights}
We have established that, with probability at least $(1-\nu)$, the following holds:
\begin{align}
\bar{\delta}(\mathbb{P}_f) - \zeta &\lesssim \underset{m'\in \mathbb{Z}^+}{\inf~}\sqrt{\frac{log\frac{1}{\nu} + m'|\mc{Y}|log(2)}{2m} + \delta'} + 2\delta'
\end{align}
where ${\delta' = \frac{k(m')}{\sqrt{m'}}}$ and we can usually take ${\zeta \approx 0}$ (as we can make this arbitrarily small with a large enough network, due to \cite{hornik1989multilayer} and lemma \ref{lemma:sanity} if we train on cross-entropy errors). $k(m')$ is trivially less than or equal to $m'$, but it is generally going to be quite small since it is dependent on a statement only requiring the \textit{existence of functions} satisfying an empirical deviation bound. This is in contrast to classical statistical learning theory bounds which instead require \textit{for all functions} statements of the same sort. Furthermore, $k(m')$ is not strictly increasing with model complexity. On the contrary, $k(m')$ can decrease as the hypothesis space grows (given that we maintain ${\mc{W}}$ continuity), since having more functions will increase the probability of such existences. By Theorem \ref{thm:rate_asy}, we can also assume that ${\frac{k(m')}{m'} \to 0}$ as ${m' \to 0}$. These intuitions tell us that the decomposition in Theorem \ref{thm:estimated_info_bound} has successfully extracted a good amount of the problem's complexity into the term $I(X;Z)$. The primary complexity term in ${\bar{\delta}(\mathbb{P}_f)}$ - given a sufficiently complex hypothesis space -  arises from the complexity of the class variable itself.

\section{Experiments}

\subsection{How These Bounds Solve Experimental Discrepancy} \label{subsec:illustration}

We argue that the bounds presented in this paper explain the experimental discrepancy that we've alluded to a few times in this paper. These tightened, less sensitive bounds imply that, in many cases, it is simply not optimal in terms of information losses to compress a neural network's input. This can be seen visually in Figure \ref{toy1}. Here we have set up a toy classification problem with ${H(Y) = log_2(10)}$, ${H(X) = 21}$, and ${I(Y;Z^*) = H(Y)\left(1-e^{-\frac{I(X;Z^*)}{2}}\right)}$. The information quantities in this toy example are thus similar to MNIST \cite{rippel2013high}. We have plotted $I(Y;Z^*)$ along with the bounds of this paper (assuming ${\zeta \approx 0, k(m') \approx 0}$) for ${m=10,000}$, $5,000$, and $2,000$ data points. We see that very little to nothing can be gained by compression in the $m=10,000$ and $m=5,000$ cases. Serious gains can only be obtained in the ${m=2,000}$ case. On the right side of this figure, we plot the old bounds, which predicts a peak at around $5$ bits even for ${10,000}$ data points. Thus the lack of compression found experimentally on smaller datasets is explained by our new bounds, but not by the old ones. 

But if the entropy of the feature space becomes large, as we've made it for the third plot in this figure, compression becomes important even with our new bounds. This helps to explain why neural networks seem to yield compression on `harder' datasets, but do not on `easier' ones. 

\begin{figure}[t]
\centering
\resizebox{0.23\textwidth}{!}{
\begin{tikzpicture}[baseline, scale=1]
\datavisualization [ scientific axes=clean,
					y axis={grid, label=$I(Y;\hat{Z})$},
                    x axis={label=$I(X;\hat{Z})$},
					visualize as smooth line/.list={a,b,c,d},
					style sheet=strong colors,
					style sheet=vary dashing,
                    legend={south east inside, rows=2},
					a={label in legend={text=$I(Y;Z^*)$}},
					b={label in legend={text=$10000$}},
					c={label in legend={text=$5000$}},
                    d={label in legend={text=$2000$}},
					data/format=function]
data [set=a] {
	var x : interval [0:21];
	func y = 3.32-3.32*exp(-\value x/5);
}
data [set=b] {
	var x : interval [0:21];
	func y = 3.32-3.32*exp(-\value x/5)-2*0.01*\value x - 0.2;
}
data [set=c] {
	var x : interval [0:21];
	func y = 3.32-3.32*exp(-\value x/5)-2*0.02*\value x - 0.26;
}
data [set=d] {
	var x : interval [0:21];
	func y = 3.32-3.32*exp(-\value x/5)-2*0.03*\value x - 0.38;
};
\end{tikzpicture}}\resizebox{0.247\textwidth}{!}{
\begin{tikzpicture}[baseline, scale=1]
\datavisualization [ scientific axes=clean,
					y axis={grid, label=$I(Y;\hat{Z})$, max value=3.2},
                    x axis={label=$I(X;\hat{Z})$},
					visualize as smooth line/.list={a,b},
					style sheet=strong colors,
					style sheet=vary dashing,
                    legend={south inside, rows=1},
					a={label in legend={text=$I(Y;Z^*)$}},
					b={label in legend={text=$10000$}},
					data/format=function ]
data [set=a] {
	var x : interval [0:6];
	func y = 3.32-3.32*exp(-\value x/5);
}
data [set=b] {
	var x : interval [0:6];
	func y = 3.32-3.32*exp(-\value x/5)-0.03*2^(\value x);
};
\end{tikzpicture}}

\resizebox{0.247\textwidth}{!}{
\begin{tikzpicture}[baseline, scale=1]
\datavisualization [ scientific axes=clean,
					y axis={grid, label=$I(Y;\hat{Z})$, max value=3.2},
                    x axis={label=$I(X;\hat{Z})$},
					visualize as smooth line/.list={a,b},
					style sheet=strong colors,
					style sheet=vary dashing,
                    legend={south inside, rows=1},
					a={label in legend={text=$I(Y;Z^*)$}},
					b={label in legend={text=$10000$}},
					data/format=function ]
data [set=a] {
	var x : interval [0:100];
	func y = 3.32-3.32*exp(-\value x/20);
}
data [set=b] {
	var x : interval [0:100];
	func y = 3.32-3.32*exp(-\value x/20)-2*0.01*\value x - 0.2;
};
\end{tikzpicture}}
\caption{(left) New bounds on a low entropy feature space (right) Old bounds on the same space. (Bottom) New bounds on a high entropy feature space.} 
\label{toy1}
\end{figure}
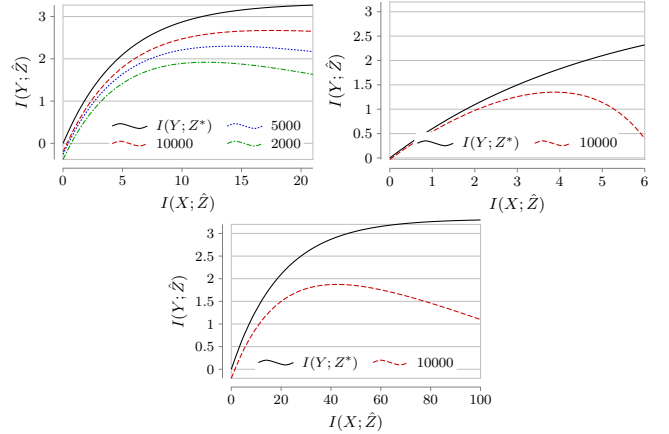

\subsection{Tightness of Bounds}\label{sec:experiments}

\begin{figure}
\centering

\resizebox{0.2\textwidth}{!}{
\begin{tikzpicture}[baseline,remember picture]
\begin{axis}[xlabel=$m$,
    title=Satimage,
	xmin=500,
    xmax=2000,
    ymin=0,
    ymax=.2,
	ylabel= $\bar{\delta}(\mathbb{P}_f)$,
	grid=both,
	minor grid style={gray!25},
	major grid style={gray!25},
	width=0.75\linewidth]

\addplot+[name path=A, blue, dotted]
                    coordinates {
                    (500, 0.134)
                    (1000, 0.11)
                    (1500, 0.094)
                    (2000, 0.082)
                    };
\addplot+[name path=C, smooth, red, dotted]
                    coordinates { 
                    (500, 0.15)
                    (1000, 0.123)
                    (1500, 0.11)
                    (2000, 0.10)
                    };
\end{axis}
\end{tikzpicture}}\resizebox{0.2\textwidth}{!}{
\begin{tikzpicture}[baseline,remember picture]
\begin{axis}[xlabel=$m$,
    title=Cifar-10,
	xmin=500,
    xmax=5000,
    ymin=0.4,
    ymax=0.6,
	ylabel= $\bar{\delta}(\mathbb{P}_f)$,
	grid=both,
	minor grid style={gray!25},
	major grid style={gray!25},
	width=0.75\linewidth]

\addplot+[name path=A, blue, dotted]
                    coordinates {
                    (500, 0.59)
                    (1000, 0.55)
                    (2500, 0.48)
                    (5000, 0.43)
                    };
\addplot+[name path=C, smooth, red, dotted]
                    coordinates { 
                    (500, 0.60)
                    (1000, 0.56)
                    (2500, 0.49)
                    (5000, 0.445)
                    };
\end{axis}
\end{tikzpicture}}

\resizebox{0.2\textwidth}{!}{
\begin{tikzpicture}[baseline,remember picture]
\begin{axis}[xlabel=$m$,
    title=Vehicle,
	xmin=50,
    xmax=200,
    ymin=0,
    ymax=0.4,
	ylabel= $\bar{\delta}(\mathbb{P}_f)$,
	grid=both,
	minor grid style={gray!25},
	major grid style={gray!25},
	width=0.75\linewidth]

\addplot+[name path=A, blue, dotted]
                    coordinates {
                    (50, 0.3)
                    (100, 0.22)
                    (150, 0.19)
                    (200, 0.18)
                    };
\addplot+[name path=C, smooth, red, dotted]
                    coordinates { 
                    (50, 0.32)
                    (100, 0.26)
                    (150, 0.23)
                    (200, 0.21)
                    };
\end{axis}
\end{tikzpicture}}\resizebox{0.2\textwidth}{!}{
\begin{tikzpicture}[baseline,remember picture]
\begin{axis}[xlabel=$m$,
    title=Credit-g,
	xmin=50,
    xmax=200,
    ymin=0.2,
    ymax=0.35,
	ylabel= $\bar{\delta}(\mathbb{P}_f)$,
	grid=both,
	minor grid style={gray!25},
	major grid style={gray!25},
	width=0.75\linewidth]

\addplot+[name path=A, blue, dotted]
                    coordinates {
                    (50, 0.3)
                    (100, 0.26)
                    (150, 0.25)
                    (200, 0.244)
                    };
\addplot+[name path=C, smooth, red, dotted]
                    coordinates {
                    (50, 0.31)
                    (100, 0.27)
                    (150, 0.257)
                    (200, 0.25)
                    };
\end{axis}
\end{tikzpicture}}

\caption{${(\bar{\delta}(\mathbb{P}_{f}) - \zeta)}$ for several datasets. (Blue) True confidence interval, (Red) bound [Theorem \ref{thm:non_asy_small}].}
\label{fig:non-asy}
\end{figure}
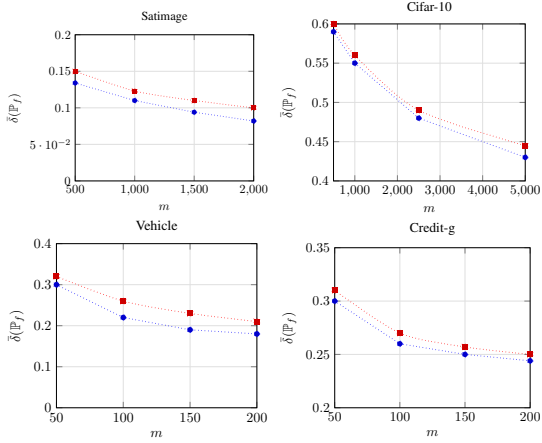

\begin{figure}
\centering

\resizebox{0.2\textwidth}{!}{
\begin{tikzpicture}[baseline,remember picture]
\begin{axis}[xlabel=$m$,
    title=MNIST (1000 units),
	xmin=500,
    xmax=10000,
    ymin=0,
    ymax=3.32,
	ylabel= $\bar{\delta}(\mathbb{P}_f)$,
	grid=both,
	minor grid style={gray!25},
	major grid style={gray!25},
	width=0.75\linewidth]

    \addplot+[name path=A, blue, dotted, ultra thin,  no marks]
                        coordinates {
                    (500, 1.21) 
                    (1000, 1.05) 
                    (2000, 0.94) 
                    (5000, 0.74) 
                    (10000, 0.59)
                    };
    \addplot+[name path=B,black, no marks, domain=0:50000] {0*x};
    \addplot[blue!15] fill between[of=A and B];
    \addplot+[name path=C, no marks, smooth, red, ultra thin, dotted]
                        coordinates { 
                    (500, 1.64)
                    (1000, 1.41)
                    (2000, 1.27)
                    (5000, 0.97)
                    (10000, 0.71)
                    };
    \addplot[red!15] fill between[of=A and C];
\end{axis}

\end{tikzpicture}}\resizebox{0.2\textwidth}{!}{
\begin{tikzpicture}[baseline,remember picture]
\begin{axis}[xlabel=$m$,
    title=MNIST (50 units),
	xmin=500,
    xmax=10000,
    ymin=0,
    ymax=3.32,
	ylabel= $\bar{\delta}(\mathbb{P}_f)$,
	grid=both,
	minor grid style={gray!25},
	major grid style={gray!25},
	width=0.75\linewidth]

    \addplot+[name path=A, blue, dotted, ultra thin,  no marks]
                        coordinates {
                    (500, 1.35) 
                    (1000, 1.18) 
                    (2000, 1.03) 
                    (5000, 0.85) 
                    (10000, 0.622)
                    };
    \addplot+[name path=B,black, no marks, domain=0:50000] {0*x};
    \addplot[blue!15] fill between[of=A and B];
    \addplot+[name path=C, no marks, smooth, red, ultra thin, dotted]
                        coordinates { 
                    (500, 1.42)
                    (1000, 1.26)
                    (2000, 1.13)
                    (5000, 0.94)
                    (10000, 0.769)
                    };
    \addplot[red!15] fill between[of=A and C];
\end{axis}

\end{tikzpicture}}

\resizebox{0.2\textwidth}{!}{
\begin{tikzpicture}[baseline,remember picture]
\begin{axis}[xlabel=$m$,
    title=Cifar-10 (1000 units),
	xmin=500,
    xmax=10000,
    ymin=0,
    ymax=4.0,
	ylabel= $\bar{\delta}(\mathbb{P}_f)$,
	grid=both,
	minor grid style={gray!25},
	major grid style={gray!25},
	width=0.75\linewidth]

    \addplot+[name path=A, blue, dotted, ultra thin,  no marks]
                        coordinates {
                    (500, 3.11) 
                    (1000, 3.03) 
                    (2000, 3.00) 
                    (5000, 2.61) 
                    (10000, 1.77)
                    };
    \addplot+[name path=B,black, no marks, domain=0:50000] {0*x};
    \addplot[blue!15] fill between[of=A and B];
    \addplot+[name path=C, no marks, smooth, red, ultra thin, dotted]
                        coordinates { 
                    (500, 3.32)
                    (1000, 3.32)
                    (2000, 3.32)
                    (5000, 3.10)
                    (10000, 2.17)
                    };
    \addplot[red!15] fill between[of=A and C];
\end{axis}

\end{tikzpicture}}\resizebox{0.2\textwidth}{!}{
\begin{tikzpicture}[baseline,remember picture]
\begin{axis}[xlabel=$m$,
    title=Cifar-10 (10000 units),
	xmin=500,
    xmax=10000,
    ymin=0,
    ymax=4.0,
	ylabel= $\bar{\delta}(\mathbb{P}_f)$,
	grid=both,
	minor grid style={gray!25},
	major grid style={gray!25},
	width=0.75\linewidth]

    \addplot+[name path=A, blue, dotted, ultra thin,  no marks]
                        coordinates {
                    (500, 3.1) 
                    (1000, 2.96) 
                    (2000, 2.80) 
                    (5000, 2.37) 
                    (10000, 1.713)
                    };
    \addplot+[name path=B,black, no marks, domain=0:50000] {0*x};
    \addplot[blue!15] fill between[of=A and B];
    \addplot+[name path=C, no marks, smooth, red, ultra thin, dotted]
                        coordinates { 
                    (500, 3.32)
                    (1000, 3.32)
                    (2000, 3.32)
                    (5000, 2.79)
                    (10000, 2.01)
                    };
    \addplot[red!15] fill between[of=A and C];
\end{axis}

\end{tikzpicture}}

\caption{${I_{Loss}^{(1)}}$ over varying architectures. (Blue) True confidence interval, (Red) Information bound [Theorem \ref{thm:estimated_info_bound}].}
\label{fig:full}
\end{figure}
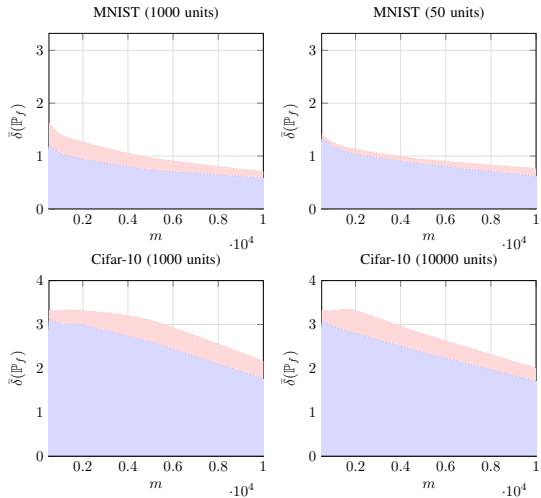

For these experiments, we have used the MINE-f \cite{belghazi2018mine} estimator of mutual information for $I(X;Z)$ quantities. We assume that ${\hat{I}(Y;\hat{Z})}$ is equal to $H(Y)$, and estimate ${I(Y;\hat{Z})}$ via validation error probability and Fano's inequality. To make the classifier representation stochastic, we used permanent dropout  with a rate of $0.7$. All classifiers are trained for $10,000$ epochs, and all information estimations are performed for $2000$ epochs. All neural networks are trained with the Adam optimizer. All models used a learning rate of ${5\times 10^{-4}}$.

We first tested the non-asymptotic bound of Theorem \ref{thm:non_asy_small} on four of the datasets provided by OpenML \cite{OpenML2013} across several training data sizes (dependent on the overall size of the dataset in question). Our classifier consisted of a neural network with a single hidden layer of $1000$ units. The results are plotted in figure \ref{fig:non-asy}. We took a confidence interval ${\nu=0.5}$ for the plot of the bound, and plotted the mean value of ten experiments for the `true' ${50\%}$ confidence interval (assuming a symmetric distribution). We estimated $k(m')$ via $k_{c}m'^r$ with ${r < \frac{1}{2}}$. In each case, we estimated $k_{c}$ and $r$ in sample for the smallest tested training data size. This, of course, only gives us a `functional behavior' experiment, but we do see that this behavior is consistent with the true values.

We then tested the bound of Theorem \ref{thm:estimated_info_bound} for MNIST and Cifar-10 using the true value of ${\bar{\delta}(\mathbb{P}_f)}$ in each case. The results are shown in Figure \ref{fig:full}. Each dataset is experimented on with a classifier given by a fully connected neural network with single hidden layer, with varying hidden layer sizes. The deviations here are to show that the bound is decent across differing architectures. The bound is quite close to the true confidence interval in each case. 

\section{Conclusion}\label{sec:conclusion}
This paper presented new bounds on information losses from finite data. This began in the form of a relationship between these losses, the expected total variation of the neural model, and the information held in the hidden representation of the feature space. Then, by bounding the total variation term without invoking any more dependence on model complexity, we obtained bounds that are much tighter and less sensitive to $I(X;Z)$ than previous theory. The paper provided applications of this theoretical framework, focusing primarily on relevant contradictory experimental work that previously went unexplained. It concluded with experiments showing that the bound presented in this paper corresponds well to experiment.

\appendix
\subsection{Proof of Lemma \ref{lemma:type_2}}
\begin{proof}
\begin{align}
I(Y; Z_{\epsilon}^*) - I(Y;\hat{Z}_{\epsilon}) &= I(Y;Z_{\epsilon}^*) - \hat{I}(Y;Z_{\epsilon}^*) \notag \\ 
&\quad + \hat{I}(Y; Z_{\epsilon}^*) - I(Y;\hat{Z}_{\epsilon}) \notag \\
&\leq K(\cdot) + \hat{I}(Y; Z_{\epsilon}^*) - I(Y;\hat{Z}_{\epsilon}) \notag \\
&\leq K(\cdot) + \hat{I}(Y; \hat{Z}_{\epsilon}) + \epsilon - I(Y;\hat{Z}_{\epsilon}) \notag \\ 
&\leq 2K(\cdot) + \epsilon
\end{align}
\end{proof}

\subsection{Proof of Lemma \ref{lemma:construction_correctness}}
\begin{proof}
We first check that the defined variables $J, U, V$ and $W$ have valid distributions. For $J$ to be valid, we need only check that $\rho < 1$. Indeed by replacing the min operation in $m_l(x,y)$ with $p_{Y|X}(y|x)$, we have 
\begin{equation}
\rho = \int \left(\sum_{y}m_l(x,y)\right) d\mathbb{P}_X \leq \int d\mathbb{P}_{XY} = 1
\end{equation}
The variable $U$ is similarly valid as can be seen as follows:
\begin{equation}
\int d\mathbb{P}_U = \frac{1}{\rho}\int \left(\sum_{y}m_l(u_1, u_2)\right) d\mathbb{P}_X = \frac{\rho}{\rho} = 1
\end{equation}
And the variables $V$ and $W$ follow similarly with ${\int d\mathbb{P}_V = \frac{1}{1-\rho}\left(\int d\mathbb{P}_{XY} - \rho \right) = 1}$, and ${\int d\mathbb{P}_W = \frac{1}{1-\rho}\left(\int d\hat{\mathbb{P}}_{XY} - \rho \right) = 1}$.

We then need to show that the marginals of the coupling satisfy ${\gamma_{\tilde{X}, \tilde{Y}, \tilde{Z}} = \mathbb{P}_{XYZ}}$ and ${\gamma_{\hat{X}, \hat{Y}, \hat{Z}} = \hat{\mathbb{P}}_{XYZ}}$. To begin, we first show that ${\gamma_{\tilde{X},\tilde{Y}}(x,y) = p_{X,Y}(x,y)}$ and that ${\gamma_{\hat{X},\hat{Y}}(x,y) = \hat{p}_{X,Y}(x,y)}$ as follows:
\begin{align}
\gamma_{\tilde{X},\tilde{Y}}(x,y) &= \rho\frac{p(x)m_l(x,y)}{\rho}  \notag\\
&\hphantom{=}+ (1-\rho)\frac{p(x)p(y|x) - p(x)m_l(x,y)}{1- \rho} \notag \\ 
&= p(x,y)
\end{align}
\begin{align}
\gamma_{\hat{X},\hat{Y}}(x,y) &= \rho\frac{p(x)m_l(x,y)}{\rho} \notag \\
&\hphantom{=}+ (1-\rho)\frac{p(x)\hat{p}(y|x) - p(x)m_l(x,y)}{1- \rho} \notag \\
&= \hat{p}(x,y).
\end{align}
Finally, since we defined $\tilde{Z}$ and $\hat{Z}$ through the distributions ${\gamma_{\tilde{Z}|\tilde{X}}(z|x) = \gamma_{\hat{Z}|\hat{X}}(z|x) = p(z|x)}$, we have
\begin{align}
\gamma_{\tilde{X}, \tilde{Y}, \tilde{Z}}(x,y,z) &=
\gamma_{\tilde{X}, \tilde{Y}}(x,y)\gamma_{\tilde{Z}|\tilde{X}}(z|x) = p(x,y)p(z|x)
\end{align}
\begin{align}
\gamma_{\hat{X}, \hat{Y}, \hat{Z}}(x,y,z) &= 
\gamma_{\hat{X}, \hat{Y}}(x,y)\gamma_{\hat{Z}|\hat{X}}(z|x) = \hat{p}(x,y)\hat{p}(z|x) 
\end{align}
\end{proof}

\subsection{Proof of Lemma \ref{lemma:max_couple_variation}}
\begin{proof}
To prove the first equality, define the following subsets of $\mathcal{Y}$.
\begin{equation}
A(x) := \{y : p(y|x) \leq \hat{p}(y|x)\}
\end{equation}
Then for any coupling of these two models, 
\begin{align}
\mathbb{P}(\tilde{Y} = \hat{Y} | X = x) &\leq \mathbb{P}(\tilde{Y} \in A(x) | X = x) \notag \\
&\hphantom{=}+ \mathbb{P}(\hat{Y} \in A^c(x) | X = x) \notag \\
&= \sum_{y \in A(x)} p(y|x) + \sum_{y \in A^c(x)}\hat{p}(y|x) \notag \\
&= \sum_y \text{min}\{p(y|x), \hat{p}(y|x) \} = \sum_y m_l(x,y)
\end{align}
It follows that:
\begin{align}
\mathbb{P}(\tilde{Y} = \hat{Y} | \tilde{X} = \hat{X}) = \int_X \mathbb{P}(\tilde{Y} = \hat{Y} | X=x) d\mathbb{P}_X \leq \rho
\end{align}
But we also have for this particular coupling, that ${\mathbb{P}(\tilde{Y} = \hat{Y} | \tilde{X} = \hat{X}) \geq P_{J}(1) = \rho}$. Thus we must have equality. 

To prove the second equality, we will use the fact that $\text{min}\{a,b\} = \frac{a + b - |a-b|}{2}$. Then 
\begin{align}
\sum_y m_l(x,y) = \frac{1}{2}\sum_y \left( p(y|x) + \hat{p}(y|x) - |p(y|x) - \hat{p}(y|x)|\right) \notag \\
= 1 - \frac{1}{2}\sum_y |p(y|x) - \hat{p}(y|x)|
\end{align}
Thus ${\rho = 1 - \mathbb{E}_{\mathbb{P}_X}\left[\frac{1}{2}\sum_y |p(y|x) - \hat{p}(y|x)|\right]}$
\end{proof}

\subsection{Proof of Lemma \ref{lemma:sanity}}
\begin{proof}
\begin{align}
\bar{\delta}(\hat{\mathbb{P}}) &= \int \delta_{TV}(\mathbb{P}_{Y|X}, \hat{\mathbb{P}}_{Y|X}) d\mathbb{P}_X \notag \\ 
& \leq \int\sqrt{\frac{1}{2}\mc{D}_{KL}\left[\mathbb{P}_{Y|X} ~||~ \hat{\mathbb{P}}_{Y|X}  \right]}d\mathbb{P}_X \notag \\
& \leq \sqrt{\int \frac{1}{2}\mc{D}_{KL}\left[\mathbb{P}_{Y|X} ~||~ \hat{\mathbb{P}}_{Y|X}  \right]d\mathbb{P}_X} \notag \\
&= \sqrt{\frac{1}{2}H_{\mathbb{P},\hat{\mathbb{P}}}(Y|X)}
\end{align}
\end{proof}

\subsection{Proof of Lemma \ref{lemma:kl_change}}
\begin{proof}
This infinum can be found by the following Lagrangian: ${\mc{L} = \mathbb{E}\left[ g\cdot (h + log~g) \right] + \lambda\left(\mathbb{E}\left[g\right] - 1 \right)}$ (we will see that we don't need to worry about the $g(\omega) \geq 0$ constraints because the solution to the lagrangian we just wrote will yield a function $g$ in which those constraints are not tight). The functional derivative of this Lagrangian is ${h(\omega) + log~g(\omega) + 1 + \lambda}$. Fixing this to zero yields ${g(\omega)=e^{-\lambda}e^{-(h(\omega)+1)}}$. Setting $\lambda$ through normalization then yields ${g(\omega)=\frac{1}{W}e^{-(h(\omega)+1)}}$ where $W=\mathbb{E}\left[ e^{-(h(\omega)+1)} \right]$. Plugging this solution into our objective yields ${-1 - log~W = -log~\mathbb{E}\left[e^{-(h(\omega)+1)} \right] - 1}$.  Since our objective function was a strictly convex functional with a positive second variation given by $\frac{1}{g(\omega)}$, this is a minimizer. 
\end{proof}

\subsection{Proof of Lemma \ref{lemma:ijensen}}
\begin{proof}
This follows from reference \cite{liao2018sharpening} (Theorem 1) with ${\phi = -log(\cdot)}$ while replacing $h(x;\mu)$ with ${\phi''(x)/2 = \frac{1}{2x^2}}$. Denote ${Y=e^{-2f^2}}$. The range of $Y$ is a subset of ${[e^{-2}, 1]}$. On this set, the supremum of ${\phi''(x)/2}$ is $\frac{1}{2}$. Thus ${log\left(\mathbb{E}\left[ Y \right] \right) \leq \mathbb{E}\left[ log(Y) \right] + \frac{1}{2}Var\left[Y\right]}$. But ${Var\left[e^{-2f^2} \right] \leq 4Var[f^2] \leq 4Var[f]}$ (because $f$ has range bounded by $[0,1]$). We thus have ${log\left(\mathbb{E}\left[ e^{-2f^2} \right] \right) \leq -2\mathbb{E}\left[f^2 \right] + 2Var\left[f \right]}$. This completes the proof since ${Var\left[f \right] = \mathbb{E}[f^2] - \mathbb{E}[f]^2}$.  
\end{proof}

\bibliographystyle{ieeetran}
\bibliography{references}

\end{document}